\def\FORARXIVACL{1}
\documentclass[11pt]{article}

\newif\ifvenuetmlr    \venuetmlrfalse
\newif\ifvenueneurips \venueneuripsfalse
\newif\ifvenueblackbox \venueblackboxfalse
\newif\ifaclpreprint   \aclpreprintfalse
\newif\ifanon         \anonfalse        %
\ifdefined\FORNEURIPS\anontrue\fi        %
\ifdefined\FORBLACKBOX\venueblackboxtrue\anontrue\fi
\ifdefined\FORARXIVACL\venueblackboxtrue\aclpreprinttrue\fi

\ifvenueblackbox
  \ifaclpreprint
    \usepackage[preprint]{acl}
  \else
    \usepackage[review]{acl}
  \fi
\fi

\usepackage[utf8]{inputenc}
\usepackage[T1]{fontenc}
\ifvenueblackbox
  \usepackage{times}
  \usepackage{latexsym}
  \usepackage{inconsolata}
\else
  \usepackage{lmodern}   %
\fi
\usepackage{amsmath,amssymb,amsthm}
\usepackage{booktabs}
\usepackage{graphicx}
\usepackage{xcolor}
\usepackage{microtype}
\ifvenueblackbox
\else
  \usepackage[numbers,sort&compress]{natbib}
  \usepackage[hidelinks]{hyperref}
\fi

\newcommand{\conf}{\mathrm{conf}}
\newtheorem{proposition}{Proposition}

\title{Solver-Hard Is Not Model-Hard: A Hardness-Controlled Diagnostic
for LLM Constraint Reasoning}

\author{\ifanon
  Anonymous Authors\\
  \textit{Paper under double-blind review}
\else
  Lucky Verma\\
  Independent Researcher\\
  \texttt{luckyv1@umbc.edu}
\fi
}
\date{}

\begin{document}
\maketitle

\begin{abstract}
LLM constraint reasoners are often evaluated near the random-SAT phase transition,
confounding density and solver hardness. We test instance-level transfer while
near-matching clause density. At aligned size bins,
with near-matched density and matched maximum clause width, we compare proof-hard
expander-Tseitin and proof-easy ladder-Tseitin formulas, pigeonhole anchors, and
density-mismatched controls. Theory separates their resolution hardness; a solver-specific
Glucose mean-conflict proxy differs by up to $51\times$, and five other solvers preserve the
direction. Across three included models (243 instances each; a fourth is excluded for
abstention), the near-matched-density
accuracy gaps range from $-32$ to $+20$ points, with a pooled gap of $+1.7$ points
($p=0.74$) and a wrong-signed correctness--conflict association ($r=+0.15$).
A proof-preserving relabeling lowers accuracy in all
five clusters for one model (mean $-93$ points) but not another, exposing model-specific
surface sensitivity. In a pre-registered extension, provider-reported completion-token spend does not
consistently increase with the proxy after accounting for formula length and censoring. At
16k, the reasoning model spends more on proof-easy matched formulas and
exhausts its budget on the solver-easiest UNSAT family; the 32k C1 gap is absent.
These scoped dissociations concern verdict accuracy and observed token spend, not
certificate solving, exact proof length, or allocation efficiency.
\end{abstract}

\begin{figure*}[t]\centering
  \includegraphics[width=\textwidth]{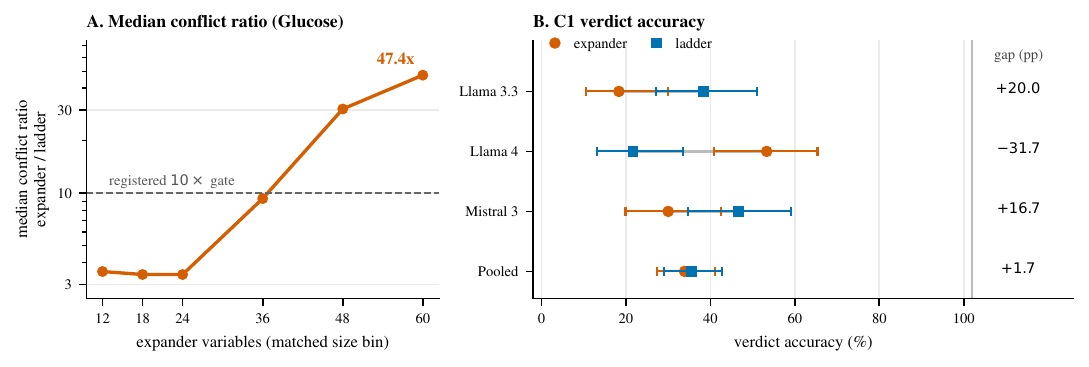}
  \includegraphics[width=\textwidth]{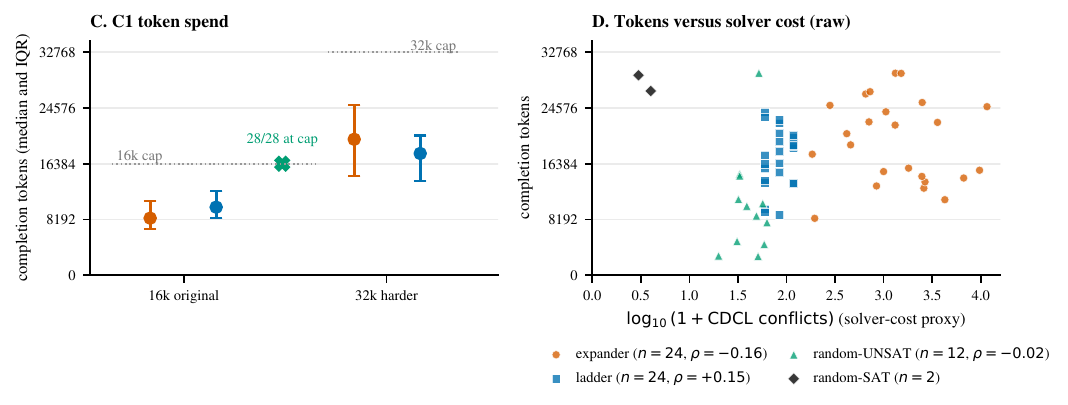}
  \caption{Four views of the solver--model dissociation. \emph{A:} at matched maximum
  width and near-matched density, the pre-specified median Glucose-conflict criterion between proof-hard
  expander and proof-easy ladder Tseitin clears $10\times$ in the two largest bins
  (Section~\ref{sec:stepa}). \emph{B:} the ladder-minus-expander accuracy gap changes
  sign across models and is $+1.7$ points pooled ($p=.74$); orange circles denote
  expander and blue squares ladder. \emph{C:} among non-abstained, uncensored rows, a reasoning model
  spends more tokens on ladder at 16k (raw $p=.0115$; post-hoc two-budget Bonferroni
  $p=.0231$), while the 32k contrast is null (raw $p=.170$). The X separately marks
  28/28 non-abstained random-UNSAT rows censored at the 16,384-token cap.
  \emph{D:} the raw 32k scatter reports per-family $n$ and Spearman $\rho$; registered
  length-controlled results are in Section~\ref{sec:compute}. Bars in B are Wilson
  95\% intervals and bars in C are IQRs. C's bars and D exclude abstained and
  budget-censored rows (C: expander/ladder $n=53/48$ at 16k, $n=24/24$ at 32k).}
  \label{fig:money}
  \label{fig:compute}
\end{figure*}

\begin{figure*}[t]\centering
  \includegraphics[width=0.92\textwidth]{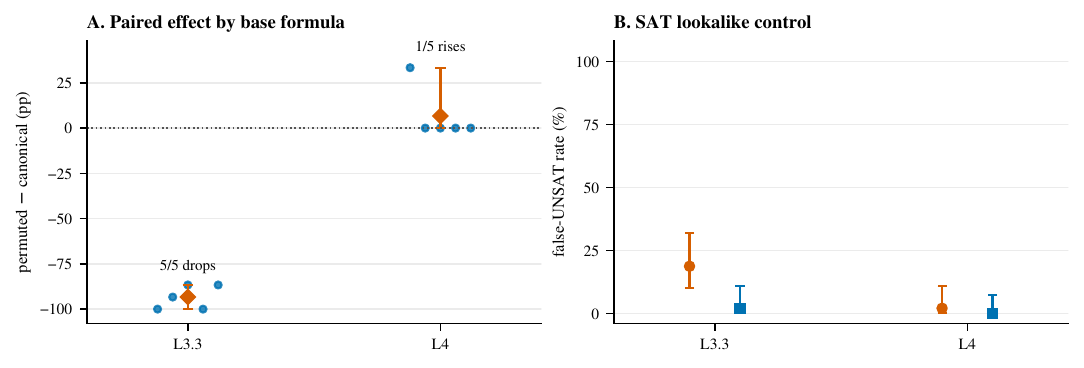}
  \caption{Surface dependence is strong for one model and absent as a common law.
  \emph{A:} each blue point is the permuted-minus-canonical accuracy effect for one of
  five base PHP formulas; orange diamonds are means and bars span the five observed
  cluster effects. \emph{B:} false-UNSAT rate on 48 satisfiable PHP-lookalikes (orange
  circles) versus 48 size-and-density-matched random-SAT controls (blue squares), with
  Wilson 95\% intervals. L3.3 is Llama~3.3 and L4 is Llama~4.}
  \label{fig:surface}
\end{figure*}

\section{Introduction}
A growing literature reports that neural reasoners fail on Boolean satisfiability (SAT)
and related constraint-satisfaction problems as instances approach the satisfiability
phase transition, and a second literature attributes such limits to model architecture,
contrasting autoregressive transformers with recurrent, diffusion, or energy-based
alternatives. Evaluations used to support either account often leave the same
density--hardness confound unresolved.

On uniformly random formulas the clause-to-variable density $\alpha = m/n$ and the
algorithmic hardness of the instance co-vary near the phase transition. A model
that uses density as a surface cue can therefore reproduce a qualitative ``fails
near the threshold'' curve without tracking instance-level search cost. Testing
whether a structural hardness axis orders model failures requires a construction
that separates it from density.

We decouple them with controlled dissociations. The headline Tseitin pair shares
maximum clause width and near-matched density while graph structure changes the known
resolution-hardness ordering: expander formulas require exponential resolution
\citep{urquhart1987,bsw2001}, whereas ladder formulas (treewidth two) have polynomial
proofs. Pigeonhole formulas provide a second exponential-resolution anchor
\citep{haken1985}, and random-UNSAT formulas provide a high-density, low-CDCL-cost
control; these auxiliary families are not density-matched pairs. The empirical hardness measure is the per-instance conflict count of
a conflict-driven clause-learning (CDCL) solver, used as a solver-specific proxy for
refutation cost rather than as the shortest proof length itself.

This design directly adjudicates two instance-level accounts: clause density and
proof-complexity transfer \citep{jiangcai2026}. Architecture-level accounts
\citep{arch_limits2507} motivate the broader question but require a cross-paradigm panel;
our all-autoregressive evaluation does not test them. Holding the model fixed while
near-matching density, matching maximum width, and intervening on graph structure supplies a
controlled test of whether the resulting proof-hardness ordering predicts model
outcomes. Graph topology necessarily co-varies with formal proof hardness, so this is
not a claim that proof length is the only changed property. The claim-bearing question
is whether LLM accuracy on these
formula families tracks proof complexity, tracks the clause-density proxy, or tracks
neither; the evaluation object is per-instance SAT/UNSAT accuracy at near-matched density and
maximum width, scored against the classical-solver oracle, not a downstream task or benchmark
leaderboard position.

\paragraph{Contributions.}
\begin{enumerate}
  \item A construct-validity diagnostic that separates proof complexity from clause density
        at matched maximum width and near-matched density. On the solver, its hard/easy separation grows to
        $51\times$ (Sections~\ref{sec:dissociation}--\ref{sec:stepa}).
  \item A pre-specified three-model refutation: the predicted accuracy ordering changes
        sign across models, while a proof-preserving relabeling exposes a model-specific
        surface-form confound (Sections~\ref{sec:phase1}--\ref{sec:hardness}).
  \item A pre-registered extension showing that observed completion-token spend also fails to
        track the solver-cost proxy across aligned size bins and near-matched density on the tested models and
        budgets (Section~\ref{sec:compute}).
\end{enumerate}

\section{The dissociation, stated precisely}\label{sec:dissociation}
We distinguish two properties of an UNSAT instance that are confounded on random formulas.

\paragraph{Clause density (the surface proxy).} For a CNF formula with $n$ variables and
$m$ clauses, $\alpha = m/n$. On random $k$-SAT, $\alpha$ controls the satisfiability
phase transition \citep{ding2015,mezard2009} and is readable from the formula without
search. It is the variable a pattern-matcher can exploit.

\paragraph{Proof complexity (the structural variable).} The length of the shortest
resolution refutation of an UNSAT formula. By the size--width relation
\citep{bsw2001}, proof size is exponential in the resolution width, which is in turn
controlled by combinatorial expansion of the underlying structure. We use the
CDCL conflict count of Glucose as a solver-specific empirical proxy: CDCL learns
clauses through resolution, but its realized conflict count is neither the shortest
resolution proof nor an implementation-independent complexity measure.

On random formulas these two co-vary, which is why they cannot be told apart there.
We engineer families in which they are decoupled, and form three contrasts.

\paragraph{C1 (near-matched-density expansion contrast, the headline).} Tseitin formulas
encode a parity constraint per graph vertex; the formula is UNSAT when the total
charge is odd. On a $3$-regular expander, refutation requires exponential resolution
\citep{urquhart1987,bsw2001}; on a ladder graph (treewidth two) it is polynomial.
Both encodings share maximum clause width three, label, and, by construction, nearly
identical clause density; the ladder's four degree-two corner vertices contribute eight
width-two clauses per instance (20\% of its clauses in the evaluated size range), while
every expander clause has width three. The designed intervention is graph expansion,
which changes known resolution hardness; topology, including this clause-length
difference, is therefore a co-varying structural feature visible in the formula. Prediction: if this proof-hardness ordering transfers, accuracy on ladder
exceeds accuracy on expander at matched size.

Writing $C_s(F)$ for solver $s$'s conflict count, we keep the descriptive mean ratio
and the pre-specified median statistic distinct:
\[
\begin{aligned}
R^{\mathrm{mean}}_s(b)
  &=\frac{\mathbb{E}[C_s(F)\mid F\in\text{expander}(b)]}
          {\mathbb{E}[C_s(F)\mid F\in\text{ladder}(b)]},\\
R^{\mathrm{med}}_s(b)
  &=\frac{\operatorname{med}_{F\in\text{expander}(b)}C_s(F)}
          {\operatorname{med}_{F\in\text{ladder}(b)}C_s(F)}.
\end{aligned}
\]
The $10\times$ apparatus criterion applies to $R^{\mathrm{med}}$; the mean ratio is
descriptive.
For model $M$, the pre-specified transfer estimand is
$\Delta_M=A_M(\text{ladder})-A_M(\text{expander})$. The hypothesis requires
$R^{\mathrm{mean}}_s(b),R^{\mathrm{med}}_s(b)>1$ on the solver and $\Delta_M>0$ on the model; Figure~\ref{fig:money} makes
their opposing empirical behavior visible.

\paragraph{C2 (density is anti-predictive).} The Tseitin expander family has density
$\alpha \approx 2.67$; a random-UNSAT family at $\alpha \approx 6$ has low Glucose
conflict counts at the sizes we test. Density ranks the high-$\alpha$ random family as
harder, whereas the solver-cost proxy ranks the low-$\alpha$ expander as harder.
Prediction: accuracy on the
expander is lower than on the denser random-UNSAT family.

\paragraph{C3 (pigeonhole anchor).} Pigeonhole $\mathrm{PHP}^{n+1}_n$ requires
exponential resolution \citep{haken1985} and provides a second proof-hard family with
no graph structure, contrasted against random-UNSAT at matched variable count.

\paragraph{Falsifiable prediction and gate.} If failure is governed by proof
complexity rather than clause density, then (i) the C1 accuracy gap is positive and
large at near-matched density, (ii) the C2 gap shows density inverting, and (iii) per-instance
correctness is negatively associated with $\log_{10}(\conf)$ pooled across families.
The author-recorded, pre-specified decision rule requires, in at least two of three models, a C1 gap
$\geq 20$ percentage points with either non-overlapping Wilson 95\% confidence intervals
or a two-proportion $p<.05$. It additionally requires a significant negative
correctness--$\log_{10}(\conf)$ association pooled across families and treats
near-zero random-SAT accuracy as an uninformative UNSAT-bias failure. A gap below $10$
points rejects the pre-specified transfer prediction. The implemented scorer narrows the
association to UNSAT rows to avoid the SAT/conflict-zero contrast and checks its sign rather
than significance. We disclose both departures; neither can rescue the hypothesis because
the UNSAT-only coefficient is positive and no model clears the full transfer gate.

\section{Apparatus}\label{sec:apparatus}
\paragraph{Families.} The near-matched-density contrast pairs UNSAT Tseitin formulas on
3-regular expanders (exponential resolution hardness) with Tseitin
formulas on ladder graphs (treewidth two, polynomial proofs). Pigeonhole formulas provide
a second proof-hard anchor. Random UNSAT formulas at $\alpha\approx6$ provide a
high-density, low-conflict control, and random SAT formulas at $\alpha\approx3$ measure
verdict bias. Ground truth is mechanical. All instances use held-out seeds
($\geq1000$) and carry generator, solver, and code provenance; Appendix
Table~\ref{tab:families} lists the full roles.

\paragraph{Proof-complexity proxy.} Every instance is solved by Glucose~4.2 (via
PySAT); we record the conflict, decision, and propagation counts from the solver's
accumulated statistics. Conflicts are the primary proxy because CDCL clause learning
can be represented as resolution; the realized count remains solver- and
implementation-dependent and is not the shortest refutation size. The classical solver
also serves as the exact label oracle on these generated instances and as a runtime
baseline.

\paragraph{Near-matched-density construction.} Each expander instance uses a connected random
3-regular graph on $V$ vertices ($1.5V$ edge variables, $4V$ clauses,
$\alpha\approx 2.67$), sampled from an ensemble that is expanding with high
probability; we do not certify per-instance expansion. The ladder control uses
\texttt{ladder\_graph}\allowbreak$(V/2)$ at the same maximum width and nearly the
same density. Charges are random with an enforced odd sum, which forces UNSAT on the
connected graph. Random families are matched to the Tseitin variable counts. Because
\texttt{ladder\_graph}$(m)$ yields about two fewer variables than the matched
3-regular expander, the C1 contrast pairs families by size bin
($\lfloor n/6 \rceil$), not by exact variable count.

\paragraph{Shared harness.} The autoregressive arm uses one frozen prompt and parser
(released as \path{scripts/prompt_harness.py}) across every family:
each instance is rendered as DIMACS CNF, the model is asked to decide satisfiability
and emit \texttt{VERDICT: SAT}/\texttt{UNSAT} at temperature zero, and the verdict line
is parsed; a response without a parseable verdict line is an abstention. The primary
endpoint is therefore verdict accuracy, not certificate-verified
SAT solving. Hosted-model calls used NVIDIA NIM's OpenAI-compatible API; this serving
backend is part of the protocol but not an internal-compute measurement. Using the
identical harness makes the comparison a property of the instances rather than a
family-specific prompt. The API and solver arms log JSON with provenance;
no checkpoints are produced.

No predictive model is fitted, so there is no train/test split or learned decision
threshold. All protocol-specified generated instances are evaluation rows with generator seeds
held out from apparatus development. The solver supplies the binary SAT/UNSAT target label;
within each included-model cohort an abstention remains in the denominator and is scored as
incorrect. The pre-specified contrast thresholds determine the paper-level gate, while
per-family uncertainty is reported with Wilson 95\% confidence intervals.

\paragraph{Reproducibility.} The public release at
\ifanon
an anonymized repository (URL withheld for review)
\else
\url{https://github.com/lucky-verma/solver-hard-is-not-model-hard}
\fi
contains the
generator, CDCL scorer, LLM harness, pinned statistics receipts, per-instance
derived measurements, and claim-to-result manifest. Generated response text is
not redistributed. No model checkpoints are produced: the reproduction units are
generator seed/version pairs and provider identifiers with run timestamps. Provider
APIs do not expose immutable weight hashes, so exact weight-level reproduction depends
on provider stability.

\section{The hardness separation appears on the solver}\label{sec:stepa}
Before evaluating models, we verify that the construction separates solver cost from
clause density. Figure~\ref{fig:money}A and Appendix Table~\ref{tab:stepa-full} report
the frozen solver-scaling set (10 instances per nondeterministic cell).

At matched maximum width and near-matched density, the mean Glucose-conflict ratio between expander
and ladder Tseitin grows monotonically across the six size bins:
$3.3\times$, $3.5\times$, $3.6\times$, $10.2\times$, $28.5\times$, and
$51.4\times$. At about 60 variables, their median conflict counts are 5,550 and
117. The high-density random-UNSAT control ($\alpha=6$) is easier still (median 43.5),
despite having $2.25\times$ the expander's density. The second theoretical anchor,
pigeonhole, rises from 7 to 7,100 conflicts between 12 and 56 variables, consistent
with its established resolution lower bound; this finite sequence does not estimate
an asymptotic exponent.

A locked post-primary sensitivity repeats all 245 instances with five additional
CDCL implementations. All 1,470 labels agree, and every solver preserves the large-bin
direction: the weakest expander/ladder ratio from 36 variables onward is $7.0\times$;
the non-Glucose medians are $18.4\times$, $47.7\times$, and $165.2\times$
(Appendix Table~\ref{tab:multisolver}). Conflict scales remain solver-specific.

The pre-specified solver-side criterion is a median expander-to-ladder conflict ratio of at
least $10\times$. The protocol did not state how many scaling bins had to clear, so we
report the full curve rather than convert it into a binary confirmatory pass. Figure~\ref{fig:money}A
plots the six median ratios: $3.6\times$, $3.4\times$, $3.4\times$, $9.3\times$, $30.4\times$,
and $47.4\times$, clearing the threshold only at the two largest bins. The descriptive
mean-ratio series reported above crosses $10\times$ from the 36-variable bin onward.
The protocol's PHP--random-UNSAT criterion also clears at large size: 7,100 versus
43.5 median conflicts at 56 and 60 variables, respectively.
Thus the
apparatus supplies the ordering the model test requires: density stays approximately fixed while the
solver-specific refutation-cost proxy diverges. The model test asks whether accuracy
inherits that ordering at the smaller sizes where verdict accuracy remains measurable.

\section{The dissociation does not transfer to LLMs}\label{sec:phase1}
We evaluate Llama~3.3 70B Instruct, Llama~4 Maverick 17B-128E Instruct, and
Mistral Large~3 675B Instruct at temperature zero; exact provider identifiers are
listed in Appendix~\ref{sec:repro-app}.
Each model receives the same 243 instances: 60 from each large family and three
pigeonhole anchors. Abstentions remain in the denominator and count as incorrect.
We report Wilson 95\% intervals, two-sample proportion tests for the pre-specified
contrasts, and point-biserial correlation between correctness and
$\log_{10}(1+\conf)$. The per-model gate is primary; pooled $p$-values aggregate
response rows descriptively and do not support population inference over models.
A fourth model, \texttt{gpt-oss-120b}, abstains on 189/243
instances (77.8\%) and is excluded from the primary pooled summaries under a predeclared
abstain-dominated contingency. Including it yields a pooled $+2.5$-point C1 gap
($p=.55$) and zero of four models passing the gate (Section~\ref{sec:limits}).

\begin{table}[t]\centering\footnotesize\setlength{\tabcolsep}{3.2pt}
\caption{Verdict accuracy by family. C1 is ladder minus expander; parenthetical values
are raw two-sided $p$-values. Full Wilson intervals are in Appendix
Table~\ref{tab:phase1-full}.}
\label{tab:phase1}
\begin{tabular}{@{}lrrr@{}}
\toprule
family & Llama 3.3 & Llama 4 & Mistral 3 \\
\midrule
random-SAT & .92 & 1.00 & .38 \\
expander & .18 & .53 & .30 \\
ladder & .38 & .22 & .47 \\
random-UNSAT & .27 & .00 & .02 \\
PHP ($n{=}3$) & 1.00 & .33 & 1.00 \\
\midrule
C1 gap & +.20 (.015) & $-$.32 ($<$.001) & +.17 (.060) \\
\bottomrule
\end{tabular}
\end{table}

\paragraph{The near-matched-density ordering does not transfer (C1).}
The ladder-minus-expander gap is $+20$ points for Llama~3.3, $-32$ for Llama~4,
and $+17$ for Mistral~3 (Figure~\ref{fig:money}B). Pooled, it is $+1.7$ points
($p=0.74$). Two models point in the predicted direction, but neither satisfies the
pre-specified magnitude, interval, and correlation conditions jointly; the third reverses
the effect significantly. No included or excluded model reopens the transfer account.

\paragraph{Neither alternative hardness proxy orders accuracy.}
The densest UNSAT family, random-UNSAT ($\alpha=6$), is answered worst or near-worst
by every model, even though it is the easiest UNSAT family for Glucose. Conversely,
the proof-hard PHP anchor is answered most accurately by two models. PHP contains only
three instances per model and is corroborating rather than load-bearing. Across all
UNSAT rows, correctness has a wrong-signed pooled association with solver cost
($r=+0.15$); the per-model values are $+0.03$, $+0.28$, and $+0.16$. Thus neither
clause density nor the conflict proxy induces a consistent error ordering.

\paragraph{Satisfiability status is the dominant pooled contrast.}
The two low-abstention models score .92 and 1.00 on the satisfiable control, while
their larger UNSAT cells are substantially weaker. Mistral~3 scores .38 on that control
because it abstains on 58\% of those prompts; its pooled abstention rate is 29.2\%,
versus below 0.5\% for the two Llama models. The pooled included-model abstention rate
is 9.9\% (72/729). This is consistent with a witness-availability asymmetry, but a
verdict-only endpoint cannot establish that a model constructed or checked a satisfying
assignment. The remaining family differences motivate the proof-preserving surface test
in Section~\ref{sec:hardness}.

\section{Solver hardness does not transfer to model hardness}\label{sec:hardness}
The experiment tests one transfer assumption: an ordering that separates a SAT solver
should also order a model's verdict accuracy on the same instances. Glucose exhibits a
growing, $51\times$ near-matched-density separation; the three LLMs exhibit a sign-changing
accuracy contrast and no negative correctness--conflict association. In this regime,
solver-hard is not model-hard.

\paragraph{A proof-preserving surface intervention.}
To probe the residual family effects, we pre-specified a second control on pigeonhole
formulas. For each $n\in\{3,\ldots,7\}$, we retain one canonical serialization and create
15 versions by bijectively relabeling variables and shuffling clauses and literals. These
operations preserve the formula, satisfiability, and minimum resolution length/width
while changing only its serialized surface (Appendix Proposition~\ref{prop:surface-invariance}).
We treat the base formula, not each rendering, as the
independent unit. Llama~3.3 drops in all five paired clusters by 86.7--100 points
(mean $-93.3$; raw accuracies 1.00 versus .067). Llama~4 changes only for $n=3$ and is
zero in the other four clusters (mean $+6.7$; raw .00 versus .067). Thus a logically
irrelevant surface change can alter one model's verdicts sharply, but it does not induce a
common law across models (Figure~\ref{fig:surface}A).

A pre-specified SAT control reaches the same boundary. We compare satisfiable pigeonhole
lookalikes with random-SAT formulas matched on variable count and density. For Llama~3.3,
the false-UNSAT rate is .188 versus .021 over 48 instances per family (a $+16.7$-point
shift, $p=.0075$); for Llama~4 the shift is only $+2.1$ points. The result therefore
misses the pre-specified five-model causal-recognizability gate ($\geq25$ points with
non-overlapping Wilson intervals in at least two of five planned models); neither completed
model qualifies. We conclude surface dependence, not template recognition as a uniform mechanism
(Figure~\ref{fig:surface}B).

\paragraph{What the negative rules out.}
On these families, models, and scales, neither clause density nor the Glucose-conflict
proxy is a reliable predictor of verdict accuracy. This is a boundary on those predictors,
not a claim that density or proof structure is irrelevant to every formula distribution or
scale. Graph topology necessarily co-varies with the formal hardness intervention, and the
primary endpoint is a verdict rather than a certificate.

\paragraph{Architecture remains unresolved.}
All included systems are API-served autoregressive transformers. Their disagreement rules
out a common empirical ordering in this panel, but it neither identifies nor refutes a
shared architectural mechanism. Statistical-physics and proof-complexity theories still
describe the instances and classical search procedures
\citep{gamarnik2021,mezard2009,krzakala2007}; Section~\ref{sec:stepa} verifies their solver-side relevance
here. The narrower result is that this ordering fails to predict LLM verdicts at the tested
scale.

\section{Token spend also dissociates}\label{sec:compute}
Accuracy leaves open a second channel: a model might spend more test-time compute on
proof-hard inputs even when its verdicts remain poor. Prior work finds that reasoning
length often tracks heuristic difficulty weakly or encourages overthinking
\citep{thoughtterminator2025,thinkhow2025,mirage2025,morethinking2026}. We instead ask
whether the provider-reported completion-token count, a coarse output-length proxy for
test-time spend rather than an internal-compute counter, tracks the validated solver-cost axis
across aligned size bins and near-matched density.

\paragraph{Registered design and censor operationalization.}
For each family we measure Spearman correlation between token count and
$\log(1+\conf)$, then partial rank correlation controlling formula length
($n_{\mathrm{clauses}}$). The matched expander--ladder contrast uses a two-sided
Mann--Whitney test. At the model-arm level, feasibility requires a parseable verdict and
reported completion-token count on at least 60\% of rows, with budget censoring on at most
30\%. The preregistration named \texttt{finish\_reason=length} but did not
specify the censoring denominator. At 16k, the provider returned no length flags even
when counts reached the cap. The pinned scorer therefore uses a disclosed post-hoc,
conservative rule (\texttt{finish\_reason=length} or at least 95\% of the cap) and divides
by token-bearing rows: 45/171 (26.3\%) are censored, versus 0/171 under the literal flag;
both pass the 30\% gate. Correlations and medians exclude abstentions and rows censored
under the post-hoc rule. The design covers the
three instruct models, Nemotron~3 Ultra 550B at
a 16,384-token budget on the original 243 instances, and a registered 32,768-token
follow-up on larger fresh instances. The latter spans expander conflict counts of
183--11,579 and random-UNSAT counts of 16--62 in the evaluated subsample. A second
reasoning lineage, R1-Distill-Qwen-32B, fails the feasibility gate with 72.8\%
abstention and is excluded. Pinned SciPy code recomputes every statistic from raw rows;
the tests and multiplicity sensitivities are enumerated in Appendix~\ref{sec:compute-app}.

\paragraph{Instruct-model spend does not consistently follow proof hardness.}
Eight of 12 family cells have $|\rho|<.3$. The only two large raw correlations, both on
ladder ($-.58$ and $-.61$), cannot be separated from formula length because ranked
conflicts and $n_{\mathrm{clauses}}$ are collinear in those cells; the registered partial
correlations are therefore undefined, not zero. No instruct model spends more on the
proof-hard matched family: Llama~3.3 has
medians 555 versus 557 tokens (expander versus ladder); Llama~4 spends less on expander
(1,222.5 versus 1,387; raw $p=.0137$, Holm $p=.0412$ over the three model contrasts);
Mistral~3 is null ($p=.825$).

\paragraph{Reasoning-model spend also dissociates.}
At 16k, Nemotron's non-censored medians are 8,385 tokens on expander and 9,998 on
ladder, reversing the proof-hardness prediction. The expander cell's raw correlation
of $.33$ becomes $-.01$ after length control; a nominal small-size correlation does not
survive the registered three-tier correction. At 32k on harder inputs, the matched
medians reverse direction (19,976 versus 17,889) but remain non-significant ($p=.170$),
and every defined length-controlled association is negative. Across the three instruct
models and two feasible Nemotron budgets, neither the length-controlled cells nor the
matched-family contrasts show a consistent positive relation between spend and the proxy.

The most severe spending occurs on the solver-easiest UNSAT family: at 16k, all 28
non-abstained random-UNSAT rows hit the 16,384-token cap and accuracy is .036, whereas
the satisfiable control is answered at 1.00. Those censored correlations are
unmeasurable, not zero. Together, the results do not show consistent tracking of this
proof-hardness proxy on the tested configurations. They do not establish
allocation inefficiency: an allocator should optimize marginal accuracy gain, which
would require paired effort interventions that we did not run
\citep{zhai2026adaptive,coda2026}.

\section{Related work}\label{sec:related}
\paragraph{SAT and controlled reasoning diagnostics.}
Prior SAT evaluations vary the clause-to-variable ratio near the phase transition
\citep{hazra2024,hazra2025phase}, clause count \citep{satbench2025}, or size and depth
across broader reasoning tasks
\citep{g4satbench2023,beyondbench2025,faithfate2023,gsmsymbolic2024,illusion2025}.
SATQuest supplies controllable, verifier-backed generation and evaluation across instance,
problem, and format axes \citep{zhao2026satquest}; we claim neither a new SAT harness nor
generic controlled SAT evaluation. On random formulas, density and search cost co-vary.
The closest comparator is \citet{jiangcai2026}, which proposes proof
complexity as a lens on UNSAT failures and reports that clause-controlled Tseitin families
with different graph structures differ substantially in difficulty; its evidence is
family-level, with clause budgets matched while variable counts, input lengths, and
clause widths vary across families, and it names instance-level regression against
solver-based hardness signals as future work. We ask that complementary
transfer question with an instance-level, six-solver-validated expander--ladder ordering
under aligned size bins, near-matched density, and matched maximum clause width; that ordering does not consistently
transfer to model accuracy. Parameterized
2-CNF work controls different structural and surface axes \citep{essebbani2026robustness},
while matched-pair work studies SAT/UNSAT asymmetry near the threshold with a
paired-formula protocol that scopes out proof complexity
\citep{matchedpair2026}.

\paragraph{Proof complexity and architecture claims.}
The construction draws on exponential resolution lower bounds for pigeonhole
\citep{haken1985} and Tseitin formulas on expanders
\citep{tseitin1968,urquhart1987}, together with the size--width relation
\citep{bsw2001}. Statistical-physics accounts describe threshold and clustering structure
in random $k$-SAT \citep{mezard2009,krzakala2007,ding2015,gamarnik2021}; they target
algorithms, optimizers, and samplers rather than LLM verdict accuracy. Neural samplers on
spin-glass distributions are adjacent but address generative sampling efficiency
\citep{ghio2024}.

A separate literature attributes symbolic-reasoning failures to architecture
\citep{arch_limits2507}, while recurrent, recursive, diffusion, and energy-based systems
report gains on Sudoku, mazes, ARC, and related tasks
\citep{hrm2025,trm2025,ired2024,consformer2025,ebt2025}. These systems are usually
evaluated on fixed tiers or size-based distribution shifts rather than proof complexity at
near-matched density. Subsequent reanalysis also shows that apparent architectural gains can
depend on augmentation or repeated guessing \citep{renliu2026}. Our all-autoregressive panel
cannot adjudicate architecture. Indeed, a constructive result shows that a decoder-only
Transformer can decide 3-SAT with backtracking and deduction through chain-of-thought
\citep{pan2025sat}; our empirical negative is not an architecture-impossibility result.
It tests only the prior construct assumption that classical instance hardness should order
model failure.

\paragraph{Test-time compute and surface form.}
Reasoning length is often treated as an adaptive resource, yet models can overthink easy
problems \citep{thoughtterminator2025,thinkhow2025}, and longer traces need not improve
accuracy \citep{mirage2025,morethinking2026}. Allocation methods therefore estimate
difficulty or per-input budget--response curves
\citep{coda2026,zhai2026adaptive}. Verbose surface form can itself induce unnecessary
reasoning \citep{ren2026costumes}, while intermediate tokens need not be faithful traces
of an internal reasoning process \citep{kambhampati2025}. We do not propose an allocator
or equate intrinsic hardness with value of compute. The extension supplies a construct
test: does observed token spend track a solver-grounded hardness proxy when size and
density are controlled?

\paragraph{Methodological kin.}
Like controlled remeasurement work on apparent capability boundaries
\citep{schaeffer2023}, the contribution changes the measurement before changing the
model. The result is a scoped failure of construct transfer, not a universal account of
reasoning difficulty.

\section{Conclusion}
We separated clause density from a solver-grounded hardness ordering before evaluating
LLMs. The control works on Glucose: near-matched-density expander and ladder formulas separate
by up to $51\times$. It does not transfer to model verdicts. The corresponding accuracy
gap changes sign across three models, pools to $+1.7$ points, and correctness never has
the predicted negative association with solver conflicts.

A proof-preserving relabeling shows why uncontrolled family comparisons are hazardous:
all five base-formula clusters move sharply for one model, while four of five are unchanged
for another. The registered token-spend extension reaches a parallel boundary; observed spend
does not consistently increase with the validated solver-cost proxy once formula length and
censoring are considered, with the strongest censoring on the solver-easiest UNSAT family.

The methodological conclusion is narrower than a theory of reasoning failure. Classical
hardness is a property of an instance--solver pair, not an automatic measure of model
difficulty. Claims that invoke it should first demonstrate transfer under a controlled
dissociation.

\section*{Limitations}\label{sec:limits}
\paragraph{Scale and hardness construct.}
The accuracy result is read at 12--30 variables, where at least some families remain
measurable for the LLMs. The solver-side separation is already present there but grows
much larger at sizes where verdict accuracy approaches a floor. The conclusion is
therefore limited to the tested, LLM-tractable regime. Moreover, Glucose conflicts are a
solver- and implementation-specific proxy, not shortest resolution length. In the
matched Tseitin pair, graph expansion is both the source of the known resolution-hardness
difference and a structural feature visible in the serialization. The intervention does
not isolate proof length from every topological property, including the clause-length
difference at the ladder's corner vertices (Section~\ref{sec:dissociation}), and the
random $3$-regular construction is expanding with high probability rather than
per-instance certified.

\paragraph{Verdicts, witnesses, and labels.}
The harness scores a final SAT/UNSAT verdict. It does not require a SAT prediction to
include a solver-verified assignment, while UNSAT has no comparably short certificate in
this setup. A correct response may therefore be a correct guess. The high satisfiable-
control accuracy of the low-abstention models is consistent with witness availability but
does not establish that a witness was constructed or verified. Ground-truth labels are
mechanical, so this is not a human-label uncertainty; it is an endpoint-validity limit.

\paragraph{Model coverage and API reproducibility.}
The included panel contains three API-served autoregressive models from two model
families, not an architecture-diverse sample. Mistral~3 is retained despite its 29.2\%
abstention because most responses remain substantive. The fourth completed model,
\texttt{gpt-oss-120b}, is excluded from primary pooled summaries under a predeclared
abstain-dominated contingency after 77.8\% abstention. A four-model sensitivity leaves
zero models passing the gate. Provider identifiers and
timestamps are recorded, but immutable checkpoint hashes are unavailable, so exact
weight-level reproduction depends on provider stability.

\paragraph{Surface dependence is model-specific.}
The relabeling experiment has only five independent base-formula clusters, each with one
canonical and 15 permuted renderings. A post-hoc unit-of-analysis repair finds a decline
in all five Llama~3.3 clusters, but five clusters cannot reach conventional two-sided
sign-test significance ($p=.0625$); the second model changes in only one cluster. The
satisfiable-lookalike control likewise misses
the pre-specified cross-model causal-recognizability gate. These experiments identify
surface form as a material confound; they do not show that template recognition explains
all residual family differences.

\paragraph{Completion-token spend is a coarse, censored measure.}
Only Nemotron passes the reasoning-model feasibility gates; R1-Distill-Qwen-32B is
abstain-dominated. Reported completion tokens include the final answer and do not expose
internal computation; they are an output-length proxy for test-time spend. Cells
at the budget cap (especially all 28 non-abstained random-UNSAT rows at 16k) have
unmeasurable correlations rather than observed null effects. Larger budgets or other
reasoning models could produce different coupling. Finally, proof hardness is not the
same as the marginal value of additional compute. Because we do not run each instance
under paired effort interventions, the token results cannot estimate allocation regret
or identify an optimal policy.

\paragraph{Claim boundary.}
These results do not establish general claims about LLM reasoning, proof structure, or
architecture. At this scale, the tested ordering and conflict proxy predict neither
verdict accuracy on these families nor spontaneous token spend on feasible configurations.

\ifvenueblackbox
\else
  \bibliographystyle{plainnat}
\fi
\bibliography{references}

\begin{thebibliography}{38}
\providecommand{\natexlab}[1]{#1}

\bibitem[{Ben-Sasson and Wigderson(2001)}]{bsw2001}
Eli Ben-Sasson and Avi Wigderson. 2001.
\newblock \href {https://doi.org/10.1145/375827.375835} {Short proofs are
  narrow---resolution made simple}.
\newblock \emph{Journal of the ACM}, 48(2):149--169.

\bibitem[{Ding et~al.(2015)Ding, Sly, and Sun}]{ding2015}
Jian Ding, Allan Sly, and Nike Sun. 2015.
\newblock \href {https://doi.org/10.1145/2746539.2746619} {Proof of the
  satisfiability conjecture for large $k$}.
\newblock In \emph{Proceedings of the 47th Annual ACM Symposium on Theory of
  Computing}, pages 59--68. Association for Computing Machinery.

\bibitem[{Du et~al.(2024)Du, Mao, and Tenenbaum}]{ired2024}
Yilun Du, Jiayuan Mao, and Joshua~B. Tenenbaum. 2024.
\newblock \href {https://proceedings.mlr.press/v235/du24f.html} {Learning
  iterative reasoning through energy diffusion}.
\newblock In \emph{Proceedings of the 41st International Conference on Machine
  Learning}, volume 235 of \emph{Proceedings of Machine Learning Research},
  pages 11764--11776. PMLR.

\bibitem[{Dziri et~al.(2023)Dziri, Lu, Sclar, Li, Jiang, Lin, Welleck, West,
  Bhagavatula, Le~Bras, Hwang, Sanyal, Ren, Ettinger, Harchaoui, and
  Choi}]{faithfate2023}
Nouha Dziri, Ximing Lu, Melanie Sclar, Xiang~(Lorraine) Li, Liwei Jiang,
  Bill~Yuchen Lin, Sean Welleck, Peter West, Chandra Bhagavatula, Ronan
  Le~Bras, Jena Hwang, Soumya Sanyal, Xiang Ren, Allyson Ettinger, Zaid
  Harchaoui, and Yejin Choi. 2023.
\newblock \href {https://doi.org/10.52202/075280-3081} {Faith and fate: Limits
  of transformers on compositionality}.
\newblock In \emph{Advances in Neural Information Processing Systems},
  volume~36, pages 70293--70332. Curran Associates, Inc.

\bibitem[{Es-sebbani et~al.(2026)Es-sebbani, Marquer, Salhi, and
  Bouraoui}]{essebbani2026robustness}
Na{\"i}m Es-sebbani, Esteban Marquer, Yakoub Salhi, and Zied Bouraoui. 2026.
\newblock \href {https://doi.org/10.48550/arXiv.2602.12665} {Evaluating
  robustness of reasoning models on parameterized logical problems}.
\newblock \emph{arXiv preprint arXiv:2602.12665}.

\bibitem[{Gamarnik(2021)}]{gamarnik2021}
David Gamarnik. 2021.
\newblock \href {https://doi.org/10.1073/pnas.2108492118} {The overlap gap
  property: A topological barrier to optimizing over random structures}.
\newblock \emph{Proceedings of the National Academy of Sciences},
  118(41):e2108492118.

\bibitem[{Ghio et~al.(2024)Ghio, Dandi, Krzakala, and Zdeborov{\'a}}]{ghio2024}
Davide Ghio, Yatin Dandi, Florent Krzakala, and Lenka Zdeborov{\'a}. 2024.
\newblock \href {https://doi.org/10.1073/pnas.2311810121} {Sampling with flows,
  diffusion, and autoregressive neural networks from a spin-glass perspective}.
\newblock \emph{Proceedings of the National Academy of Sciences},
  121(27):e2311810121.

\bibitem[{Ghosal et~al.(2025)Ghosal, Chakraborty, Reddy, Lu, Wang, Manocha,
  Huang, Ghavamzadeh, and Bedi}]{mirage2025}
Soumya~Suvra Ghosal, Souradip Chakraborty, Avinash Reddy, Yifu Lu, Mengdi Wang,
  Dinesh Manocha, Furong Huang, Mohammad Ghavamzadeh, and Amrit~Singh Bedi.
  2025.
\newblock \href
  {https://proceedings.neurips.cc/paper_files/paper/2025/hash/fc067ac218430c409d6f65403328f740-Abstract-Conference.html}
  {Does thinking more always help? mirage of test-time scaling in reasoning
  models}.
\newblock In \emph{Advances in Neural Information Processing Systems},
  volume~38, pages 172664--172691. Curran Associates, Inc.

\bibitem[{Gladstone et~al.(2025)Gladstone, Nanduru, Islam, Han, Ha, Chadha, Du,
  Ji, Li, and Iqbal}]{ebt2025}
Alexi Gladstone, Ganesh Nanduru, Md~Mofijul Islam, Peixuan Han, Hyeonjeong Ha,
  Aman Chadha, Yilun Du, Heng Ji, Jundong Li, and Tariq Iqbal. 2025.
\newblock \href {https://doi.org/10.48550/arXiv.2507.02092} {Energy-based
  transformers are scalable learners and thinkers}.
\newblock \emph{arXiv preprint arXiv:2507.02092}.

\bibitem[{Haken(1985)}]{haken1985}
Armin Haken. 1985.
\newblock \href {https://doi.org/10.1016/0304-3975(85)90144-6} {The
  intractability of resolution}.
\newblock \emph{Theoretical Computer Science}, 39:297--308.

\bibitem[{Hazra et~al.(2024)Hazra, Venturato, Zuidberg Dos~Martires, and
  De~Raedt}]{hazra2024}
Rishi Hazra, Gabriele Venturato, Pedro Zuidberg Dos~Martires, and Luc De~Raedt.
  2024.
\newblock \href {https://doi.org/10.48550/arXiv.2408.07215} {Can large language
  models reason? a characterization via 3-sat}.
\newblock \emph{arXiv preprint arXiv:2408.07215}.

\bibitem[{Hazra et~al.(2025)Hazra, Venturato, Zuidberg Dos~Martires, and
  De~Raedt}]{hazra2025phase}
Rishi Hazra, Gabriele Venturato, Pedro Zuidberg Dos~Martires, and Luc De~Raedt.
  2025.
\newblock \href {https://openreview.net/forum?id=MPTlWIVSMU} {Have large
  language models learned to reason? a characterization via 3-sat}.
\newblock In \emph{Second Conference on Language Modeling}.

\bibitem[{Jiang and Cai(2026)}]{jiangcai2026}
Tao Jiang and Shaowei Cai. 2026.
\newblock \href {https://openreview.net/forum?id=rv1SR32kl9} {Beyond clause
  count: A study of proof-relevant difficulty in {LLM} {SAT} reasoning}.
\newblock In \emph{ICLR 2026 Workshop on Logical Reasoning of Large Language
  Models}.

\bibitem[{Jolicoeur-Martineau(2025)}]{trm2025}
Alexia Jolicoeur-Martineau. 2025.
\newblock \href {https://doi.org/10.48550/arXiv.2510.04871} {Less is more:
  Recursive reasoning with tiny networks}.
\newblock \emph{arXiv preprint arXiv:2510.04871}.

\bibitem[{Kambhampati et~al.(2026)Kambhampati, Valmeekam, Bhambri, Palod,
  Saldyt, Stechly, Samineni, Kalwar, and Biswas}]{kambhampati2025}
Subbarao Kambhampati, Karthik Valmeekam, Siddhant Bhambri, Vardhan Palod, Lucas
  Saldyt, Kaya Stechly, Soumya~Rani Samineni, Durgesh Kalwar, and Upasana
  Biswas. 2026.
\newblock \href {https://icml.cc/virtual/2026/poster/67077} {Position: Stop
  anthropomorphizing intermediate tokens as reasoning/thinking traces!}
\newblock In \emph{Proceedings of the 43rd International Conference on Machine
  Learning}.

\bibitem[{Krzakala et~al.(2007)Krzakala, Montanari, Ricci-Tersenghi, Semerjian,
  and Zdeborov{\'a}}]{krzakala2007}
Florent Krzakala, Andrea Montanari, Federico Ricci-Tersenghi, Guilhem
  Semerjian, and Lenka Zdeborov{\'a}. 2007.
\newblock \href {https://doi.org/10.1073/pnas.0703685104} {Gibbs states and the
  set of solutions of random constraint satisfaction problems}.
\newblock \emph{Proceedings of the National Academy of Sciences},
  104(25):10318--10323.

\bibitem[{Li et~al.(2024)Li, Guo, and Si}]{g4satbench2023}
Zhaoyu Li, Jinpei Guo, and Xujie Si. 2024.
\newblock \href {https://openreview.net/forum?id=7VB5db72lr} {{G4SATBench}:
  Benchmarking and advancing {SAT} solving with graph neural networks}.
\newblock \emph{Transactions on Machine Learning Research}.

\bibitem[{Liu et~al.(2026)Liu, Li, Ma, Zhang, and Guo}]{thinkhow2025}
Yongjiang Liu, Haoxi Li, Xiaosong Ma, Jie Zhang, and Song Guo. 2026.
\newblock \href {https://doi.org/10.18653/v1/2026.acl-long.1766} {Think how to
  think: Mitigating overthinking with autonomous difficulty cognition in large
  reasoning models}.
\newblock In \emph{Proceedings of the 64th Annual Meeting of the Association
  for Computational Linguistics (Volume 1: Long Papers)}, pages 38105--38126.
  Association for Computational Linguistics.

\bibitem[{M{\'e}zard and Montanari(2009)}]{mezard2009}
Marc M{\'e}zard and Andrea Montanari. 2009.
\newblock \href {https://doi.org/10.1093/acprof:oso/9780198570837.001.0001}
  {\emph{Information, Physics, and Computation}}.
\newblock Oxford University Press.

\bibitem[{Mirzadeh et~al.(2025)Mirzadeh, Alizadeh-Vahid, Shahrokhi, Tuzel,
  Bengio, and Farajtabar}]{gsmsymbolic2024}
Iman Mirzadeh, Keivan Alizadeh-Vahid, Hooman Shahrokhi, Oncel Tuzel, Samy
  Bengio, and Mehrdad Farajtabar. 2025.
\newblock \href
  {https://proceedings.iclr.cc/paper_files/paper/2025/hash/ec2e7a896f8250986b3907f57621ce94-Abstract-Conference.html}
  {{GSM-Symbolic}: Understanding the limitations of mathematical reasoning in
  large language models}.
\newblock In \emph{The Thirteenth International Conference on Learning
  Representations}.

\bibitem[{Pan et~al.(2025)Pan, Ganesh, Abernethy, Esposo, and Lee}]{pan2025sat}
Leyan Pan, Vijay Ganesh, Jacob Abernethy, Chris Esposo, and Wenke Lee. 2025.
\newblock \href {https://proceedings.mlr.press/v267/pan25d.html} {Can
  transformers reason logically? a study in {SAT} solving}.
\newblock In \emph{Proceedings of the 42nd International Conference on Machine
  Learning}, volume 267 of \emph{Proceedings of Machine Learning Research},
  pages 47632--47671. PMLR.

\bibitem[{Pu et~al.(2025)Pu, Saxon, Hua, and Wang}]{thoughtterminator2025}
Xiao Pu, Michael Saxon, Wenyue Hua, and William~Yang Wang. 2025.
\newblock \href {https://openreview.net/forum?id=oHR862dpMC}
  {{THOUGHTTERMINATOR}: Benchmarking, calibrating, and mitigating overthinking
  in reasoning models}.
\newblock In \emph{Second Conference on Language Modeling}.

\bibitem[{Ren et~al.(2026)Ren, Zhang, Chen, Shen, Li, Zhang, Cao, and
  Ji}]{ren2026costumes}
Junnan Ren, Yan Zhang, Qian Chen, Yunhang Shen, Ke~Li, Shengchuan Zhang,
  Liujuan Cao, and Rongrong Ji. 2026.
\newblock \href {https://icml.cc/virtual/2026/poster/62755} {When simple
  problems wear complex costumes: Improving efficiency in {LRM}'s adaptive
  reasoning}.
\newblock In \emph{Proceedings of the 43rd International Conference on Machine
  Learning}.

\bibitem[{Ren and Liu(2026)}]{renliu2026}
Zirui Ren and Ziming Liu. 2026.
\newblock \href {https://doi.org/10.48550/arXiv.2601.10679} {Are your reasoning
  models reasoning or guessing? a mechanistic analysis of hierarchical
  reasoning models}.
\newblock \emph{arXiv preprint arXiv:2601.10679}.

\bibitem[{Schaeffer et~al.(2023)Schaeffer, Miranda, and Koyejo}]{schaeffer2023}
Rylan Schaeffer, Brando Miranda, and Sanmi Koyejo. 2023.
\newblock \href {https://doi.org/10.52202/075280-2425} {Are emergent abilities
  of large language models a mirage?}
\newblock In \emph{Advances in Neural Information Processing Systems},
  volume~36, pages 55565--55581. Curran Associates, Inc.

\bibitem[{Shojaee et~al.(2025)Shojaee, Mirzadeh, Alizadeh~Vahid, Horton,
  Bengio, and Farajtabar}]{illusion2025}
Parshin Shojaee, Iman Mirzadeh, Keivan Alizadeh~Vahid, Maxwell Horton, Samy
  Bengio, and Mehrdad Farajtabar. 2025.
\newblock \href
  {https://proceedings.neurips.cc/paper_files/paper/2025/hash/9b26ad15462c81548c0689188d2e8018-Abstract-Conference.html}
  {The illusion of thinking: Understanding the strengths and limitations of
  reasoning models via the lens of problem complexity}.
\newblock In \emph{Advances in Neural Information Processing Systems},
  volume~38, pages 108018--108059. Curran Associates, Inc.

\bibitem[{Srivastava et~al.(2026)Srivastava, Hussain, Bi, Roy, Pitre, Lu,
  Ziyadi, and Wang}]{beyondbench2025}
Gaurav Srivastava, Aafiya Hussain, Zhenyu Bi, Swastik Roy, Priya Pitre, Meng
  Lu, Morteza Ziyadi, and Xuan Wang. 2026.
\newblock \href {https://iclr.cc/virtual/2026/poster/10007593} {{BeyondBench}:
  Contamination-resistant evaluation of reasoning in language models}.
\newblock In \emph{The Fourteenth International Conference on Learning
  Representations}.

\bibitem[{Tseitin(1968)}]{tseitin1968}
G.~S. Tseitin. 1968.
\newblock \href {https://www.mathnet.ru/eng/znsl2268} {On the complexity of
  proof in propositional calculus}.
\newblock In \emph{Studies in Constructive Mathematics and Mathematical Logic,
  Part II}, volume~8 of \emph{Zapiski Nauchnykh Seminarov LOMI}, pages
  234--259. Nauka, Leningrad Department.

\bibitem[{Urquhart(1987)}]{urquhart1987}
Alasdair Urquhart. 1987.
\newblock \href {https://doi.org/10.1145/7531.8928} {Hard examples for
  resolution}.
\newblock \emph{Journal of the ACM}, 34(1):209--219.

\bibitem[{Wang et~al.(2025)Wang, Li, Sun, Chen, Liu, Wu, Lu, Song, and
  Abbasi-Yadkori}]{hrm2025}
Guan Wang, Jin Li, Yuhao Sun, Xing Chen, Changling Liu, Yue Wu, Meng Lu, Sen
  Song, and Yasin Abbasi-Yadkori. 2025.
\newblock \href {https://doi.org/10.48550/arXiv.2506.21734} {Hierarchical
  reasoning model}.
\newblock \emph{arXiv preprint arXiv:2506.21734}.

\bibitem[{Wei et~al.(2025)Wei, Wu, Wan, Suresh, Tan, Zhou, Koyejo, Wang, and
  Aiken}]{satbench2025}
Anjiang Wei, Yuheng Wu, Yingjia Wan, Tarun Suresh, Huanmi Tan, Zhanke Zhou,
  Sanmi Koyejo, Ke~Wang, and Alex Aiken. 2025.
\newblock \href {https://doi.org/10.18653/v1/2025.emnlp-main.1716} {{SATB}ench:
  Benchmarking {LLM}s' logical reasoning via automated puzzle generation from
  {SAT} formulas}.
\newblock In \emph{Proceedings of the 2025 Conference on Empirical Methods in
  Natural Language Processing}, pages 33832--33849. Association for
  Computational Linguistics.

\bibitem[{Wu et~al.(2026)Wu, Xie, Zhang, and Xiao}]{coda2026}
Siye Wu, Jian Xie, Yikai Zhang, and Yanghua Xiao. 2026.
\newblock \href {https://doi.org/10.48550/arXiv.2603.08659} {{CODA}:
  Difficulty-aware compute allocation for adaptive reasoning}.
\newblock \emph{arXiv preprint arXiv:2603.08659}.

\bibitem[{Xu et~al.(2025)Xu, Li, Sanner, and Khalil}]{consformer2025}
Yudong Xu, Wenhao Li, Scott Sanner, and Elias~Boutros Khalil. 2025.
\newblock \href {https://proceedings.mlr.press/v267/xu25q.html}
  {Self-supervised transformers as iterative solution improvers for constraint
  satisfaction}.
\newblock In \emph{Proceedings of the 42nd International Conference on Machine
  Learning}, volume 267 of \emph{Proceedings of Machine Learning Research},
  pages 69432--69450. PMLR.

\bibitem[{Zhai et~al.(2026)Zhai, Li, Xiao, Li, and Wang}]{zhai2026adaptive}
Zhiyuan Zhai, Bingcong Li, Bingnan Xiao, Ming Li, and Xin Wang. 2026.
\newblock \href {https://doi.org/10.48550/arXiv.2604.14853} {Adaptive test-time
  compute allocation for reasoning {LLM}s via constrained policy optimization}.
\newblock \emph{arXiv preprint arXiv:2604.14853}.

\bibitem[{Zhang et~al.(2026)Zhang, Chen, and Chen}]{matchedpair2026}
Leizhen Zhang, Shuhan Chen, and Sheng Chen. 2026.
\newblock \href {https://doi.org/10.1145/3808209} {Satisfiability solving with
  llms: A matched-pair evaluation of reasoning capability}.
\newblock \emph{Proceedings of the ACM on Software Engineering},
  3(FSE):4599--4621.

\bibitem[{Zhang(2025)}]{arch_limits2507}
Zheng Zhang. 2025.
\newblock \href {https://openreview.net/forum?id=Gz5HMiJLqv} {Comprehension
  without competence: Architectural limits of llms in symbolic computation and
  reasoning}.
\newblock \emph{Transactions on Machine Learning Research}.

\bibitem[{Zhao et~al.(2026)Zhao, Li, Bo, Takezoe, Hui, Guang, {Renlei}, Qin,
  and Long}]{zhao2026satquest}
Yanxiao Zhao, Yaqian Li, Zi-Hao Bo, Rinyoichi Takezoe, Haojia Hui, Mo~Guang,
  {Renlei}, Xiaolin Qin, and Kaiwen Long. 2026.
\newblock \href {https://doi.org/10.18653/v1/2026.acl-long.96} {{SATQ}uest: A
  verifier for logical reasoning evaluation and reinforcement fine-tuning of
  {LLM}s}.
\newblock In \emph{Proceedings of the 64th Annual Meeting of the Association
  for Computational Linguistics (Volume 1: Long Papers)}, pages 2109--2131.
  Association for Computational Linguistics.

\bibitem[{Zhou et~al.(2026)Zhou, Ling, Chen, Wang, Fan, and
  Wang}]{morethinking2026}
Shu Zhou, Rui Ling, Junan Chen, Xin Wang, Tao Fan, and Hao Wang. 2026.
\newblock \href {https://doi.org/10.18653/v1/2026.findings-acl.1199} {When more
  thinking hurts: Overthinking in {LLM} test-time compute scaling}.
\newblock In \emph{Findings of the Association for Computational Linguistics:
  ACL 2026}, pages 23967--23977. Association for Computational Linguistics.

\end{thebibliography}

\clearpage
\appendix
\section{Instance families and solver validation}

\begin{table}[ht]\centering\footnotesize\setlength{\tabcolsep}{3pt}
\caption{Instance families. Ground truth is mechanical. Theoretical resolution hardness
is established for the canonical hard families; Glucose conflicts are the empirical,
solver-specific cost proxy.}
\label{tab:families}
\begin{tabular}{@{}lll@{}}
\toprule
family & label/status & role \\
\midrule
expander & UNSAT, exp. & hard, C1/C2 \\
ladder & UNSAT, poly. & easy, C1 \\
PHP & UNSAT, exp. & hard, C3 \\
random-UNSAT & UNSAT, low conflict & easy, C2/C3 \\
random-SAT & SAT, no refutation & base-rate control \\
\bottomrule
\end{tabular}
\end{table}

\paragraph{Surface-intervention invariance.}
\begin{proposition}\label{prop:surface-invariance}
Let $T$ bijectively rename a CNF formula's variables and arbitrarily reorder its clauses
and the literals within each clause. Then $F$ and $T(F)$ have the same satisfiability
status and the same minimum resolution-refutation length and width.
\end{proposition}
\begin{proof}
A satisfying assignment for $F$ composes with the inverse renaming to satisfy $T(F)$,
and conversely. A variable renaming also maps every resolution inference for $F$ to an
inference for $T(F)$ with the same clause width; reordering does not change a clause as
a set of literals. Thus any refutation maps to one of equal length and width. Applying
$T^{-1}$ gives the reverse inequalities, so both minima are equal.
\end{proof}

\begin{table}[ht]\centering\footnotesize\setlength{\tabcolsep}{1.8pt}
\caption{Full solver-scaling table: median Glucose conflicts. Nondeterministic
families contain 10 generated instances per cell; PHP has one canonical formula per
size. The ladder sizes are 10, 16, 22, 34, 46, and 58 variables and are paired to the
nearest expander size bin. Figure~\ref{fig:money}A shows the corresponding
pre-specified median-conflict ratios (Section~\ref{sec:stepa}).}
\label{tab:stepa-full}
\begin{tabular}{@{}lrrrrrr@{}}
\toprule
family ($\alpha$) & \multicolumn{6}{c}{median conflicts at $\approx n$ variables}\\
 & 12 & 18 & 24 & 36 & 48 & 60 \\
\midrule
expander & 28.5 & 68.5 & 113 & 550 & 2,551 & 5,550 \\
ladder & 8 & 20 & 33 & 59 & 84 & 117 \\
rnd-UNSAT & 6.5 & 7.5 & 11 & 16.5 & 25.5 & 43.5 \\
rnd-SAT & 1 & 0 & 2.5 & 1.5 & 2 & 2 \\
\midrule
PHP & 7 & 30$^{a}$ & 157$^{b}$ & 888$^{c}$ & -- & 7,100$^{d}$ \\
\bottomrule
\end{tabular}

\smallskip
\footnotesize $^{a}$20 variables; $^{b}$30 variables; $^{c}$42 variables; $^{d}$56 variables.
\end{table}

\begin{table}[ht]\centering\footnotesize\setlength{\tabcolsep}{3.5pt}
\caption{Post-primary multi-solver sensitivity: within-solver mean-conflict ratios
(expander/ladder) on the frozen rows. All 1,470 solver-instance labels agree. Raw
conflict scales are not compared across implementations.}
\label{tab:multisolver}
\begin{tabular}{@{}lrrr@{}}
\toprule
solver & $\sim$36 vars & $\sim$48 vars & $\sim$60 vars \\
\midrule
Glucose 4.2 & 10.2 & 28.5 & 51.4 \\
CaDiCaL 1.9.5 & 24.1 & 73.3 & 215.4 \\
MiniSat 2.2 & 11.2 & 47.7 & 227.4 \\
MapleChrono & 30.6 & 112.3 & 105.9 \\
Lingeling & 18.4 & 43.5 & 165.2 \\
MergeSat 3 & 7.0 & 18.6 & 32.2 \\
\bottomrule
\end{tabular}
\end{table}

\section{Full verdict-accuracy intervals}

Table~\ref{tab:phase1-full} reports per-family verdict accuracies with Wilson 95\%
intervals for the three included models; each block ends with its pre-specified C1
gap and raw two-sided $p$-value.

\begin{table}[t]\centering\footnotesize\setlength{\tabcolsep}{4pt}
\caption{Per-family verdict accuracy with Wilson 95\% confidence intervals. Abstentions
remain in the denominator.}
\label{tab:phase1-full}
\begin{tabular}{@{}lr@{}}
\toprule
model/family & accuracy [95\% CI] \\
\midrule
\multicolumn{2}{@{}l}{\emph{Llama 3.3}} \\
\quad random-SAT & .92 [.82, .96] \\
\quad expander / ladder & .18 [.11, .30] / .38 [.27, .51] \\
\quad random-UNSAT / PHP & .27 [.17, .39] / 1.00 [.44, 1.00] \\
\quad C1 gap & +.20 ($p=.015$) \\
\midrule
\multicolumn{2}{@{}l}{\emph{Llama 4}} \\
\quad random-SAT & 1.00 [.94, 1.00] \\
\quad expander / ladder & .53 [.41, .65] / .22 [.13, .34] \\
\quad random-UNSAT / PHP & .00 [.00, .06] / .33 [.06, .79] \\
\quad C1 gap & $-$.32 ($p=.00034$) \\
\midrule
\multicolumn{2}{@{}l}{\emph{Mistral 3}} \\
\quad random-SAT & .38 [.27, .51] \\
\quad expander / ladder & .30 [.20, .43] / .47 [.35, .59] \\
\quad random-UNSAT / PHP & .02 [.00, .09] / 1.00 [.44, 1.00] \\
\quad C1 gap & +.17 ($p=.060$) \\
\bottomrule
\end{tabular}
\end{table}

\section{Token-spend tests and sensitivities}\label{sec:compute-app}

All Spearman and tie-corrected Mann--Whitney $p$-values are SciPy asymptotic values.
The preregistered censor rule named \texttt{finish\_reason=length} and left its denominator
unstated. Because the 16k provider returned no such flags at the cap, the pinned analysis
adds a post-hoc threshold at 95\% of the cap and divides by token-bearing rows. This marks
45/171 16k rows as censored (26.3\%); the literal flag marks zero, so either rule passes
the 30\% feasibility threshold. Reported correlations and medians exclude both abstentions
and rows censored by the post-hoc rule.
The instruct-model expander--ladder analysis contains three registered tests and uses
Holm adjustment. The Nemotron cross-budget view contains two expander--ladder contrasts;
the manuscript reports both raw values and a conservative two-test Bonferroni sensitivity.
The 32k within-family correlation family contains four registered tests and uses a
four-test Bonferroni sensitivity.

In the instruct-model ladder cells, ranked conflicts and ranked clause count are collinear.
The registered partial correlations are therefore undefined; floating-point residuals from
regressing one rank vector on the other are not interpretable coefficients.

At 16k, the Nemotron expander correlation is $\rho=.33$ before and $-.01$ after the
formula-length control. Within fixed size tiers, the only nominal association is at 12
variables ($\rho=.48$, raw $p=.0355$), which does not survive the three-tier Holm
correction ($p=.107$) and does not recur at 18 or 24 variables ($-.01$ and $.04$).
On the 32k set, treating capped random-UNSAT rows as observed at the cap gives
$\rho=-.507$ (raw $p=.0190$, four-test Bonferroni $p=.0761$); excluding those rows gives
$\rho=-.0175$ ($p=.957$). The cell is therefore reported as censored, not as a negative
association.

An unregistered character-count sensitivity gives a matched-family difference in the
expander direction (raw $p=.000553$), but characters per token also differ by family
(1.98 versus 1.59). Character count consequently conflates provider-token compute with
output style and is not used as evidence for a compute gap.

\section{Reproduction inventory}\label{sec:repro-app}

The deterministic release maps manuscript claims to pinned receipts:
\begin{itemize}\setlength{\itemsep}{0pt}\raggedright
  \item \path{data/results/solver_scaling_summary.json} (solver validation);
  \item \path{data/results/primary_three_model_summary.json} (accuracy);
  \item \path{data/results/surface_cluster_summary.json} (clustered surface
        control);
  \item \path{data/results/multi_solver_summary.json} and
        \path{data/results/multi_solver_rows.jsonl} (solver sensitivity);
  \item \path{data/results/token_spend_summary.json} (token-spend tests).
\end{itemize}
The release also contains the instance generators, DIMACS renderer, Glucose
scorer, provider-response parser, figure scripts, held-out seed bands, and a
claim-to-source manifest. API calls cannot be replayed against immutable
weights when providers do not publish checkpoint hashes. The released
per-instance derived measurements and mechanical labels remain re-analyzable;
generated response text is not redistributed.

\paragraph{Protocol provenance.}
The Phase~1 and surface-control protocols are author-recorded; we cannot
show that they preceded their results, so we do not call them
preregistrations. The token-spend protocol was recorded before its result
collection. The release preserves every protocol file alongside its
results.

The included verdict-model identifiers are
\begin{itemize}\setlength{\itemsep}{0pt}
  \item \texttt{llama-3.3-70b-instruct};
  \item \texttt{llama-4-maverick-17b-128e-instruct};
  \item \texttt{mistral-large-3-675b-instruct-2512}.
\end{itemize}

\end{document}